\documentclass{article} 
\usepackage{iclr2017_conference,times}
\usepackage{hyperref}
\usepackage{url}

\usepackage{graphicx, xcolor} 
\usepackage{subfigure} 
\usepackage{caption}

\usepackage{natbib}
\usepackage{amsmath,amsfonts,amsthm,amssymb}

\usepackage{stmaryrd}
\usepackage{wrapfig}
\usepackage{multirow}
\usepackage[mathscr]{euscript}
\usepackage[shortlabels]{enumitem}
\usepackage{array}
\usepackage{diagbox}
\usepackage{float}
\usepackage{cleveref}

\usepackage[utf8]{inputenc} 
\usepackage[T1]{fontenc}    
\usepackage{booktabs}       
\usepackage{nicefrac}       
\usepackage{microtype}      

\usepackage{verbatim}

\usepackage{algorithm}
\usepackage{algorithmic}

\title{Optimal Binary Autoencoding with Pairwise Correlations}

\author{Akshay Balsubramani 
\thanks{Most of the work was done as a PhD student at UC San Diego. Now reachable at: \texttt{abalsubr@stanford.edu}.} 
\\
\texttt{abalsubr@ucsd.edu} 
}

\newtheorem{theorem}{Theorem}

\newtheorem{proposition}[theorem]{Proposition}

\newtheorem{assumption}{Assumption}

\newcommand{\vB}{\mathbf{B}}

\newcommand{\vb}{\mathbf{b}}
 
\newcommand{\ve}{\mathbf{e}}

\newcommand{\vE}{\mathbf{E}}

\newcommand{\vx}{\mathbf{x}}

\newcommand{\vG}{\mathbf{G}}

\newcommand{\vR}{\mathbf{R}}

\newcommand{\vu}{\mathbf{u}}
\newcommand{\vU}{\mathbf{U}}

\newcommand{\vw}{\mathbf{w}}
\newcommand{\vW}{\mathbf{W}}
\newcommand{\vX}{\mathbf{X}}

\newcommand{\vzero}{\mathbf{0}}
\newcommand{\vone}{\mathbf{1}}

\newcommand{\vtildex}{\tilde{\mathbf{x}}}

\DeclareMathOperator*{\argmin}{arg\,min}

\DeclareMathOperator{\sgn}{sgn}

\newcommand{\RR}{\mathbb{R}}      
\newcommand{\vnorm}[1]{\left\lVert#1\right\rVert} 
\newcommand{\abs}[1]{\left| #1 \right|}

\newcommand{\encode}{\textsc{Enc}}

\newcommand{\algname}{\textsc{PC-AE}} 
 
\newcommand{\cL}{\mathcal{L}}

\newcommand{\cO}{\mathcal{O}}
\newcommand{\lrp}[1]{\left(#1\right)}
\newcommand{\lrb}[1]{\left[#1\right]}
\newcommand{\lrsetb}[1]{\left\{#1\right\}}

\begin{document}

\maketitle

\begin{abstract}
We formulate learning of a binary autoencoder as a biconvex optimization problem which learns from the pairwise correlations between encoded and decoded bits. Among all possible algorithms that use this information, ours finds the autoencoder that reconstructs its inputs with worst-case optimal loss. The optimal decoder is a single layer of artificial neurons, emerging entirely from the minimax loss minimization, and with weights learned by convex optimization. All this is reflected in competitive experimental results, demonstrating that binary autoencoding can be done efficiently by conveying information in pairwise correlations in an optimal fashion. 
\end{abstract}

\section{Introduction}
\label{sec:intro}

Consider a general autoencoding scenario, in which an algorithm learns a compression scheme 
for independently, identically distributed (i.i.d.) $V$-dimensional bit vector data $\lrsetb{ \vx^{(1)}, \dots, \vx^{(n)} }$. 
For some encoding dimension $H$, 
the algorithm encodes each data example $\vx^{(i)} = (x_1^{(i)} , \dots, x_V^{(i)} )^\top$ into an $H$-dimensional representation $\ve^{(i)}$, 
with $H < V$. 
It then decodes each $\ve^{(i)}$ back into a reconstructed example $\vtildex^{(i)}$ using some small amount of additional memory, 
and is evaluated on the quality of the reconstruction by the cross-entropy loss commonly used to compare bit vectors. 
A good autoencoder learns to compress the data into $H$ bits so as to reconstruct it with low loss. 

When the loss is squared reconstruction error and the goal is to compress data in $\RR^V$ to $\RR^H$,  
this is often accomplished with principal component analysis (PCA), 
which projects the input data on the top $H$ eigenvectors of their covariance matrix (\cite{BK88, BH89}). 
These eigenvectors in $\RR^V$ constitute $VH$ real values of additional memory needed to decode the compressed data in $\RR^H$ 
back to the reconstructions in $\RR^V$, which are linear combinations of the eigenvectors. 
Crucially, this total additional memory does not depend on the amount of data $n$, making it applicable when data are abundant. 


This paper considers a similar problem, except using bit-vector data and the cross-entropy reconstruction loss. 
Since we are compressing samples of i.i.d. $V$-bit data into $H$-bit encodings, 
a natural approach is to remember the pairwise statistics: 
the $VH$ average \emph{correlations} between pairs of bits in the encoding and decoding, 
constituting as much additional memory as the eigenvectors used in PCA. 
The decoder uses these along with the $H$-bit encoded data, to produce $V$-bit reconstructions. 

We show how to efficiently learn the autoencoder with the worst-case optimal loss in this scenario, 
without any further assumptions, parametric or otherwise. 
It has some striking properties.

The decoding function is identical in form to the one used in a standard binary autoencoder with one hidden layer (\cite{BCV13}) and cross-entropy reconstruction loss. 
Specifically, each bit $v$ of the decoding is the output of a logistic sigmoid artificial neuron of the encoded bits, 
with some learned weights $\vw_v \in \RR^H$. 
This form emerges as the uniquely optimal decoding function, and is \emph{not} assumed as part of any explicit model.

The worst-case optimal reconstruction loss suffered by the autoencoder is convex in 
these decoding weights $\vW = \lrsetb{\vw_v}_{v \in [V]}$, and in the encoded representations $\vE$. 
Though it is not jointly convex in both, the situation still admits a natural and efficient optimization algorithm 
in which the loss is alternately minimized in $\vE$ and $\vW$ while the other is held fixed. 
The algorithm is practical, learning incrementally from minibatches of data in a stochastic optimization setting.


\subsection{Notation}
The decoded and encoded data can be written in matrix form, representing bits as $\pm 1$: 
\begin{align}
\label{eq:defofe}
\vX = 
 \begin{pmatrix}
   x_1^{(1)} & \cdots & x_1^{(n)} \\
   \vdots    & \ddots &  \vdots  \\
   x_V^{(1)}  & \cdots &  x_V^{(n)}
 \end{pmatrix}  \in [-1, 1]^{V \times n}
 \quad , \quad
\vE = 
 \begin{pmatrix}
   e_1^{(1)} & \cdots & e_1^{(n)} \\
   \vdots    & \ddots &  \vdots  \\
   e_H^{(1)}  & \cdots &  e_H^{(n)}
 \end{pmatrix}  \in [-1, 1]^{H \times n}
 \end{align}
Here the encodings are allowed to be randomized, 
represented by values in $[-1,1]$ instead of just the two values $\{ -1, 1\}$; 
e.g. $e_i^{(1)} = \frac{1}{2}$ is $+1\;\text{w.p.}\; \frac{3}{4} $ and $-1\; \text{w.p.}\; \frac{1}{4} $. 
The data in $\vX$ are also allowed to be randomized, which loses hardly any generality for reasons discussed later (\Cref{sec:datarandomized}). 
We write the columns of $\vX, \vE$ as $\vx^{(i)} , \ve^{(i)}$ for $i \in [n]$ (where $[s] := \lrsetb{1,\dots,s}$), 
representing the data. 
The rows are written as $\vx_v = ( x_v^{(1)} , \dots, x_v^{(n)} )^\top$ for $v \in [V]$ 
and $\ve_h = ( e_h^{(1)} , \dots, e_h^{(n)} )^\top$ for $h \in [H]$.

We also consider the correlation of each bit $h$ of the encoding with each decoded bit $v$ over the data, i.e. 
$b_{v,h} := \frac{1}{n} \sum_{i=1}^n x_v^{(i)} e_h^{(i)}$. 
This too can be written in matrix form as $\vB := \frac{1}{n} \vX \vE^\top \in \RR^{V \times H}$, 
whose rows and columns we respectively write as $\vb_v = ( b_{v,1} , \dots, b_{v,H} )^\top$ over $v \in [V]$
and $\vb_h = ( b_{1,h} , \dots, b_{V,h} )^\top$ over $h \in [H]$; the indexing will be clear from context. 


As alluded to earlier, the loss incurred on example $i \in [n]$ is 
the cross-entropy between the example $\vx^{(i)}$
and its reconstruction $\vtildex^{(i)}$, in expectation over the randomness in $\vx^{(i)}$. 
Defining $\ell_{\pm} (\tilde{x}_v^{(i)} ) = \ln \lrp{\frac{2}{1 \pm \tilde{x}_v^{(i)} }}$ (the \emph{partial losses} to true labels $\pm 1$), 
the loss is written as:
\begin{align}
\label{eq:defxentloss}
\ell (\vx^{(i)} , \vtildex^{(i)} ) := \sum_{v=1}^V \lrb{ \lrp{\frac{1 + x_v^{(i)} }{2} } \ell_{+} (\tilde{x}_v^{(i)} ) + \lrp{\frac{1 - x_v^{(i)} }{2} } \ell_{-} (\tilde{x}_v^{(i)} ) }
\end{align}

In addition, define a \emph{potential well} 
$\Psi (m) := \ln \lrp{1 + e^{m}} + \ln \lrp{1 + e^{-m}}$ 
with derivative $\Psi' (m) := \frac{1 - e^{-m}}{1 + e^{-m}}$. 
Univariate functions like this are applied componentwise to matrices in this paper. 


\subsection{Problem Setup}
\label{sec:setup}

With these definitions, the autoencoding problem we address can be precisely stated as two tasks, 
encoding and decoding. 
These share only the side information $\vB$. 
Our goal is to perform these steps so as to achieve the best possible guarantee on reconstruction loss, with no further assumptions. 
This can be written as a zero-sum game of an autoencoding algorithm seeking to minimize loss against an adversary, 
by playing encodings and reconstructions: 

\begin{itemize}
\item
Using $\vX$, algorithm plays (randomized) encodings $\vE$, resulting in pairwise correlations $\vB$.
\item
Using $\vE$ and $\vB$, algorithm plays reconstructions $\tilde{\vX} = \lrp{ \vtildex^{(1)} ; \dots ; \vtildex^{(n)} } \in [-1, 1]^{V \times n}$.
\item
Given $\tilde{\vX}, \vE, \vB$, adversary plays $\vX$ to maximize reconstruction loss $\frac{1}{n} \sum_{i=1}^n \ell (\vx^{(i)} , \vtildex^{(i)} ) $. 
\end{itemize}


We find the autoencoding algorithm's best strategy in two parts. 
First, we find the optimal decoding of any encodings $\vE$ given $\vB$, in Section \ref{sec:optdecode}. 
Then, we use the resulting optimal reconstruction function to outline the best encoding procedure, 
i.e. one that finds the $\vE, \vB$ that lead to the best reconstruction, in Section \ref{sec:optencode}. 
Combining these ideas yields an autoencoding algorithm in Section \ref{sec:altalg} (Algorithm \ref{alg:aerealalg}), 
where its implementation and interpretation are specified. 
Further discussion and related work in Section \ref{sec:relworkdisc} are followed by 
more extensions in Section \ref{sec:extensions} and experiments in Section \ref{sec:experiments}.

\section{Optimally Decoding an Encoded Representation}
\label{sec:optdecode}

To address the problem of Section \ref{sec:setup}, we first assume $\vE$ and $\vB$ are fixed, 
and derive the optimal decoding rule given this information. 
We show in this section that the form of this optimal decoder is precisely the same as in a classical autoencoder:  
having learned a weight vector $\vw_v \in \RR^{H}$ for each $v \in [V]$, 
the $v^{th}$ bit of each reconstruction $\vtildex^i$ is expressed 
as a logistic function of a $\vw_v$-weighted combination of the $H$ encoded bits $\ve^i$ -- a logistic artificial neuron with weights $\vw_v$. 
The weight vectors are learned by convex optimization, despite the nonconvexity of the transfer functions. 

To develop this, we minimize the worst-case reconstruction error, 
where $\vX$ is constrained by our prior knowledge that $\vB = \frac{1}{n} \vX \vE^\top$, 
i.e. $\frac{1}{n} \vE \vx_v = \vb_v \;\;\forall v \in [V]$. 
This can be written as a function of $\vE$: 
\begin{align}
\label{eq:mmxloss}
\cL_{\vB}^* (\vE) := 
\min_{\vtildex^{(1)}, \dots, \vtildex^{(n)} \in [-1,1]^V} \; \max_{\substack{ \vx^{(1)}, \dots, \vx^{(n)} \in [-1,1]^V , \\ \forall v \in [V] : \;\frac{1}{n} \vE \vx_v = \vb_v }} 
\; \frac{1}{n} \sum_{i=1}^n \ell (\vx^{(i)} , \vtildex^{(i)} ) 
\end{align}

We solve this minimax problem for the optimal reconstructions played by the minimizing player in \eqref{eq:mmxloss}, 
written as $\vtildex^{(1) *}, \dots, \vtildex^{(n) *}$. 
\begin{theorem}
\label{thm:decfn}
Define the \emph{bitwise slack function} 
$
\gamma^{\vE} (\vw, \vb) 
:= 
- \vb^\top \vw + \frac{1}{n} \sum_{i=1}^n \Psi (\vw^\top \ve^{(i)} )
$,  
which is convex in $\vw$. 
W.r.t. any $\vb_v$, this has minimizing weights 
$ \displaystyle
\vw_v^* := \vw_v^* (\vE, \vB) := \argmin_{ \vw \in \RR^H } \;\gamma^{\vE} (\vw, \vb_v)
$. 
Then the minimax value of the game \eqref{eq:mmxloss} is 
$\displaystyle \cL_{\vB}^* (\vE) = \frac{1}{2} \sum_{v=1}^V \;\gamma^{\vE} (\vw_v^* , \vb_v)$.  
For any example $i \in [n]$, 
the minimax optimal reconstruction can be written for any bit $v$ as 
$
\tilde{x}_v^{(i)*} := \frac{1 - e^{- \vw_v^{* \top} \ve^{(i)} } }{1 + e^{- \vw_v^{* \top} \ve^{(i)} }}
$.
\end{theorem}


This tells us that the optimization problem of finding the minimax optimal reconstructions $\vtildex^{(i)}$ 
is extremely convenient in several respects. 
The learning problem decomposes over the $V$ bits in the decoding, reducing to solving for a weight vector $\vw_v^* \in \RR^H$ for each bit $v$, 
by optimizing each bitwise slack function. 
Given the weights, the optimal reconstruction of any example $i$ can be specified by a layer of 
logistic sigmoid artificial neurons of its encoded bits, with $\vw_v^{* \top} \ve^{(i)}$ as the bitwise logits.

Hereafter, we write $\vW \in \RR^{V \times H}$ as the matrix of decoding weights, with rows $\lrsetb{\vw_v}_{v=1}^V$. 
In particular, the optimal decoding weights $\vW^* (\vE, \vB)$ are the matrix with rows $\lrsetb{\vw_v^* (\vE, \vB)}_{v=1}^V$.


\section{Learning an Autoencoder}
\label{sec:learn}

\subsection{Finding an Encoded Representation}
\label{sec:optencode}

Having computed the optimal decoding function in the previous section given any $\vE$ and $\vB$, 
we now switch perspectives to the encoder, which seeks to compress the input data $\vX$ into encoded representations 
$\vE$ (from which $\vB$ is easily calculated to pass to the decoder). 
We seek to find $(\vE, \vB)$ to ensure the lowest worst-case reconstruction loss after decoding; 
recall that this is $\cL_{\vB}^* (\vE)$ from \eqref{eq:mmxloss}. 

Observe that $\frac{1}{n} \vX \vE^\top = \vB$, and that the encoder is given $\vX$. 
Therefore, in terms of $\vX$, 
\begin{align}
\label{eq:rewriteloss}
\cL_{\vB}^* (\vE) 
= \frac{1}{2 n} \sum_{i=1}^n \sum_{v=1}^V \; \lrb{ - x_v^{(i)} (\vw_v^{* \top} \ve^{(i)} ) + \Psi (\vw_v^{* \top} \ve^{(i)} ) } 
:= \cL (\vW^*, \vE) 
\end{align}
by using Thm. \ref{thm:decfn} and substituting $\vb_v = \frac{1}{n} \vE \vx_v \;\;\forall v \in [V]$. 
So it is convenient to define the \emph{bitwise feature distortion}
\footnote{Noting that $\Psi (\vw_v^{ \top} \ve) \approx \abs{\vw_v^{ \top} \ve}$, we see that 
$\beta_v^{\vW} (\ve, \vx) \approx \vw_v^{ \top} \ve \lrp{\sgn(\vw_v^{ \top} \ve) - x_v}$, 
so $\vw_v^{ \top} \ve$ is encouraged to match signs with $x_v$, motivating the name.} 
for any $v \in [V]$ with respect to $\vW$, between any example $\vx$ and its encoding $\ve$: 
\begin{align}
\label{eq:bitwisedist}
\beta_v^{\vW} (\ve, \vx) := - x_v \vw_v^{ \top} \ve + \Psi (\vw_v^{ \top} \ve)
\end{align}

From the above discussion, the best $\vE$ given any decoding $\vW$, written as $\vE^* (\vW)$, solves the minimization  
\begin{align*}
\min_{ \vE \in [-1,1]^{H \times n} } \cL (\vW, \vE)
= \frac{1}{2 n} \sum_{i=1}^n \min_{ \ve^{(i)} \in [-1,1]^H } \sum_{v=1}^V \; \beta_v^{\vW} (\ve^{(i)}, \vx^{(i)} ) 
\end{align*}


which immediately yields the following result. 

\begin{proposition}
\label{prop:encfn}
Define the optimal encodings for decoding weights $\vW$ as 
$ \displaystyle 
\vE^* (\vW) := \argmin_{ \vE \in [-1,1]^{H \times n} } \; \cL (\vW, \vE)
$. 
Each example $\vx^{(i)} \in [-1,1]^{V}$ can be separately encoded as $\ve^{(i) *} (\vW)$, 
with its optimal encoding minimizing its total feature distortion over the decoded bits w.r.t. $\vW$: 
\begin{align}
\label{eq:defofenc}
\encode (\vx^{(i)} ; \vW) := \ve^{(i) *} (\vW) := \argmin_{\ve \in [-1,1]^{H}} \; \sum_{v=1}^V \; \beta_v^{\vW} (\ve, \vx^{(i)} ) 
\end{align}
\end{proposition}


Observe that the encoding function $\encode (\vx^{(i)} ; \vW)$ can be efficiently computed to desired precision since 
the feature distortion $\beta_v^{\vW} (\ve, \vx^{(i)} )$ of each bit $v$ is convex and Lipschitz in $\ve$; 
an $L_1$ error of $\epsilon$ can be reached in $\cO(\epsilon^{-2})$ linear-time first-order optimization iterations. 
Note that the encodings need not be bits, and can be e.g. unconstrained $\in \RR^{H}$ instead; 
the proof of Thm. \ref{thm:decfn} assumes no structure on them.

\subsection{An Autoencoder Learning Algorithm}
\label{sec:altalg}

Our ultimate goal is to minimize the worst-case reconstruction loss. 
As we have seen in \eqref{eq:mmxloss} and \eqref{eq:defofenc}, it is convex in the encoding $\vE$ and in the decoding parameters $\vW$, 
each of which can be fixed while minimizing with respect to the other. 
This suggests a learning algorithm that alternately performs two steps: 
finding encodings $\vE$ that minimize $\cL (\vW, \vE)$ as in \eqref{eq:defofenc} with a fixed $\vW$, 
and finding decoding parameters $\vW^* (\vE, \vB)$, 
as given in \Cref{alg:aerealalg}.

\begin{algorithm}
   \caption{Pairwise Correlation Autoencoder ($\algname$)}
   \label{alg:aerealalg}
\begin{algorithmic}
   \STATE {\bfseries Input:} Size-$n$ dataset $U$
   \STATE Initialize $\vW_0$ (e.g. with each element being i.i.d. $\sim \mathcal{N} (0,1)$)
   \FOR{$t = 1$ {\bfseries to} $T$}
   \STATE Encode each example 
   to ensure accurate reconstruction using weights $\vW_{t-1}$, 
   and compute the associated pairwise bit correlations $\vB_t$: 
   \begin{align*}
   \forall i \in [n] : [\ve^{(i)} ]_{t} = \encode (\vx^{(i)} ; \vW_{t-1})
   \qquad , \qquad
   \vB_t = \frac{1}{n} \vX \vE_{t}^\top
   \end{align*}
   \vspace{1mm}
   \STATE Update weight vectors $[\vw_v]_{t}$ for each $v \in [V]$ to minimize slack function, using encodings $\vE_t$: 
   \begin{align*}
   \forall v \in [V] : [\vw_v]_{t} = \argmin_{ \vw \in \RR^H } \; \lrb{ - [\vb_v]_t^\top \vw + \frac{1}{n} \sum_{i=1}^n \Psi (\vw^\top \ve_t^{(i)} ) }
   \end{align*}
   \ENDFOR
   \STATE {\bfseries Output:} Weights $\vW_{T}$
\end{algorithmic}
\end{algorithm}

\subsection{Efficient Implementation}
\label{sec:efficiency}
Our derivation of the encoding and decoding functions involves no model assumptions at all, 
only using the minimax structure and pairwise statistics that the algorithm is allowed to remember. 
Nevertheless, the (en/de)coders can still be learned and implemented efficiently. 

Decoding is a convex optimization in $H$ dimensions, which can be done in parallel for each bit $v \in [V]$. 
This is relatively easy to solve in the parameter regime of primary interest when data are abundant, in which $H < V \ll n$. 
Similarly, encoding is also a convex optimization problem in only $H$ dimensions. 
If the data examples are instead sampled in minibatches of size $n$, they can be encoded in parallel, 
with a new minibatch being sampled to start each epoch $t$. 
The number of examples $n$ (per batch) is essentially only limited by $nH$, the number of compressed representations that fit in memory. 


So far in this paper, we have stated our results in the transductive setting, 
in which all data are given together a priori, 
with no assumptions whatsoever made about the interdependences between the $V$ features. 
However, $\algname$ operates much more efficiently than this might suggest. 
Crucially, the encoding and decoding tasks both depend on $n$ only to average a function of $\vx^{(i)}$ or $\ve^{(i)}$ over $i \in [n]$, 
so they can both be solved by stochastic optimization methods that use first-order gradient information, like variants of stochastic gradient descent (SGD). 
We find it remarkable that the minimax optimal encoding and decoding can be efficiently learned 
by such methods, which do not scale computationally in $n$. 
Note that the result of each of these steps involves $\Omega(n)$ outputs ($\vE$ and $\tilde{\vX}$), which are all coupled together in complex ways. 

The efficient implementation of first-order methods turns out to manipulate more intermediate gradient-related quantities with facile interpretations. 
For details, see Appendix \ref{sec:details}.

\subsection{Convergence and Weight Regularization}
\label{sec:wtreg}

As we noted previously, the objective function of the optimization is \emph{biconvex}. 
This means that under broad conditions, 
the alternating minimization algorithm we specify is an instance of \emph{alternating convex search}, 
shown in that literature to converge under broad conditions (\cite{GPK07}). 
It is not guaranteed to converge to the global optimum, but each iteration will monotonically decrease the objective function. 
In light of our introductory discussion, the properties and rate of such convergence 
would be interesting to compare to stochastic optimization algorithms for PCA, which converge efficiently under broad conditions (\cite{BDF13, S16}). 

The basic game used so far has assumed perfect knowledge of the pairwise correlations, 
leading to equality constraints $\forall v \in [V] : \;\frac{1}{n} \vE \vx_v = \vb_v$. 
This makes sense in $\algname$, where the encoding phase of each epoch gives the exact $\vB_t$ for the decoding phase. 
However, in other stochastic settings as for denoising autoencoders (see Sec. \ref{sec:dae}), 
it may be necessary to relax this constraint. 
A relaxed constraint of $\vnorm{\frac{1}{n} \vE \vx_v - \vb_v}_{\infty} \leq \epsilon$ 
exactly corresponds to an extra additive regularization term of $\epsilon \vnorm{\vw_v}_1$ on the corresponding weights 
in the convex optimization used to find $\vW$ (Appendix \ref{sec:linftyconstr}). 
Such regularization leads to provably better generalization (\cite{B98}) and is often practical to use, e.g. to encourage sparsity. 
But we do not use it for our $\algname$ experiments in this paper. 



\section{Discussion and Related Work}
\label{sec:relworkdisc}

Our approach $\algname$ is quite different from existing autoencoding work in several ways. 

First and foremost, we posit no explicit decision rule, and avoid optimizing the highly non-convex decision surface 
traversed by traditional autoencoding algorithms that learn with backpropagation. 
The decoding function, given the encodings, is a single layer of artificial neurons only because of the minimax structure of the problem 
when minimizing worst-case loss. 
This differs from reasoning typically used in neural net work (see \cite{J95}), 
in which the loss is the negative log-likelihood (NLL) of the joint probability, 
which is \emph{assumed} to follow a form specified by logistic artificial neurons and their weights. 
We instead interpret the loss in the usual direct way as the NLL of the predicted probability of the data given the visible bits, 
and avoid any assumptions on the decision rule (e.g. not monotonicity in the score $\vw_v^\top \ve^{(i)}$, or even dependence on such a score). 
This justification of artificial neurons -- as the minimax optimal decision rules given information on pairwise correlations -- 
is one of our more distinctive contributions (see Sec. \ref{sec:genlosses}). 

Note that there are no assumptions whatsoever on the form of the encoding or decoding, except on the memory used by the decoding. 
Some such restriction is necessary to rule out the autoencoder just memorizing the data, 
and is typically expressed by positing a model class of compositions of artificial neuron layers. 
We instead impose it axiomiatically by limiting the amount of information transmitted through $\vB$, which does not scale in $n$; 
but we do not restrict how this information is used. 
This confers a clear theoretical advantage, 
allowing us to attain the strongest robust loss guarantee among \emph{all possible} autoencoders that use the correlations $\vB$. 

More importantly in practice, avoiding an explicit model class means that we do not have to optimize the typically non-convex model, 
which has long been a central issue for backpropagation-based learning methods (e.g. \cite{DPGCGB14}). 
Prior work related in spirit has attempted to avoid this through convex relaxations, including for multi-layer optimization under various structural assumptions (\cite{AZS14, ZLW16}), 
and when the number of hidden units is varied by the algorithm (\cite{BRVDM05, Bach14}). 

Our approach also isolates the benefit of higher $n$ in dealing with overfitting, as 
the pairwise correlations $\vB$ can be measured progressively more accurately as $n$ increases. 
In this respect, we follow a line of research using such pairwise correlations to model arbitary higher-order structure among visible units, 
rooted in early work on (restricted) Boltzmann Machines (\cite{AHS85, Smo86, RM87, FH92}). 
More recently, theoretical algorithms have been developed with the perspective of learning from the correlations between units in a network, 
under various assumptions on the activation function, architecture, and weights, 
for both deep (\cite{ABGM14}) and shallow networks (using tensor decompositions, e.g. \cite{LSS14, JSA15}). 
Our use of ensemble aggregation techniques (from \cite{BF15, BF16}) to study these problems is anticipated in spirit by prior work as well, 
as discussed at length by \cite{Ben09} in the context of distributed representations.


\subsection{Optimality, Other Architectures and Depth}

We have established that a single layer of logistic artificial neurons is an optimal decoder, 
given only indirect information about the data through pairwise correlations. 
This is not a claim that autoencoders need only a single-layer architecture in the worst case. 
Sec. \ref{sec:optencode} establishes that the best representations $\vE$ are the solution to a convex optimization, 
with no artificial neurons involved in computing them from the data. 
Unlike the decoding function, the optimal encoding function $\encode$ cannot be written explicitly in terms of artificial neurons, 
and is incomparable to existing architectures. 
Also, the encodings are only optimal given the pairwise correlations; 
training algorithms like backpropagation, which indirectly communicate other knowledge of the input data through derivative composition, 
can certainly learn final decoding layers that outperform ours, as we see in experiments. 



In our framework so far, we explore using all the pairwise correlations between hidden and visible bits to inform learning by constraining the adversary, 
resulting in a Lagrange parameter -- a weight -- for each constraint. 
These $VH$ weights $\vW$ constitute the parameters of the optimal decoding layer, 
describing a fully connected architecture. 
If just a select few of these correlations were used, only they would constrain the adversary in the minimax problem of Sec. \ref{sec:optdecode}, 
so weights would only be introduced for them, giving rise to sparser architectures. 

Our central choices to store only pairwise correlations and minimize worst-case reconstruction loss  
play a similar regularizing role to explicit model assumptions, 
and other autoencoding methods may achieve better performance on data for which these choices are too conservative, 
by e.g. making distributional assumptions on the data. 
From our perspective, other architectures with more layers -- 
particularly highly successful ones like convolutional, recurrent, residual, and ladder networks (\cite{LBH15, HZRS15, RBHVR15}) -- 
lend the autoencoding algorithm more power 
by allowing it to measure more nuanced correlations using more parameters, which decreases the worst-case loss. 
Applying our approach with these would be interesting future work. 


Extending this paper's convenient minimax characterization to deep representations with empirical success is a very interesting open problem. 
Prior work on stacking autoencoders/RBMs (\cite{VLLBM10}) and our learning algorithm $\algname$ suggest that we could train a deep network in 
alternating forward and backward passes. 
Using this paper's ideas, the forward pass would learn the weights to each layer given the previous layer's activations (and inter-layer pairwise correlations) by minimizing the slack function, 
with the backward pass learning the activations for each layer given the weights to / activations of the next layer by convex optimization (as we learn $\vE$). 
Both passes consist of successive convex optimizations dictated by our approach, quite distinct from backpropagation, 
though they loosely resemble the wake-sleep algorithm (\cite{HDFN95}).

\subsection{Generative Applications}

Particularly recently, autoencoders have been of interest largely for their many applications beyond compression, 
especially for their generative uses. 
The most directly relevant to us involve repurposing denoising autoencoders (\cite{BYAV13}; see Sec. \ref{sec:dae}); 
moment matching among hidden and visible units (\cite{LSZ15}); 
and generative adversarial network ideas (\cite{GPM+14, MSJGF15}), 
the latter particularly since the techniques of this paper have been applied to binary classification (\cite{BF15, BF15b}). 
These are outside this paper's scope, but suggest themselves as future extensions of our approach.


\section{Extensions}
\label{sec:extensions}

\subsection{Other Reconstruction Losses}
\label{sec:genlosses}

It may make sense to use another reconstruction loss other than cross-entropy, 
for instance the expected Hamming distance between $\vx^{(i)}$ and $\vtildex^{(i)}$. 
It turns out that the minimax manipulations we use work under very broad conditions, 
for nearly any loss that additively decomposes over the $V$ bits as cross-entropy does. 
In such cases, all that is required is that the partial losses $\ell_{+} (\tilde{x}_v^{(i)}), \ell_{-} (\tilde{x}_v^{(i)})$ are monotonically 
decreasing and increasing respectively 
(recall that for cross-entropy loss, this is true as $\ell_{\pm} (\tilde{x}_v^{(i)} ) = \ln \lrp{\frac{2}{1 \pm \tilde{x}_v^{(i)} }}$); 
they need not even be convex. 
This monotonicity is a natural condition, because the loss measures the discrepancy to the true label, and holds for all losses in common use.

Changing the partial losses only changes the structure of the minimax solution in two respects: 
by altering the form of the transfer function on the decoding neurons, and the univariate potential well $\Psi$ optimized to learn the decoding weights. 
Otherwise, the problem remains convex and the algorithm is identical. 
Formal statements of these general results are in \Cref{sec:genlossfull}. 


\subsection{Denoising Autoencoding}
\label{sec:dae}

Our framework can be easily applied to learn a denoising autoencoder (DAE; \cite{VLBM08, VLLBM10}), 
which uses noise-corrupted data (call it $\hat{\vX}$) for training, and uncorrupted data for evaluation. 
From our perspective, this corresponds to leaving the learning of $\vW$ unchanged, 
but using corrupted data when learning $\vE$. 
So the minimization problem over encodings must be changed to account for the bias on $\vB$ 
introduced by the noise, so that the algorithm plays given the noisy data, but to minimize loss against $\vX$. 
This is easiest to see for zero-mean noise, for which our algorithms are completely unchanged because $\vB$ does not change after adding noise. 

Another common scenario illustrating this technique is to mask a $\rho$ fraction of the input bits uniformly at random 
(in our notation, changing $1$s to $-1$s). 
This masking noise changes each pairwise correlation $b_{v,h}$ by an amount 
$\delta_{v,h} := \frac{1}{n} \sum_{i=1}^n ( \hat{x}_v^{(i)} - x_v^{(i)} ) e_h^{(i)}$, so the optimand 
Eq. \eqref{eq:rewriteloss} must therefore be modified by subtracting this factor. 
$\delta_{v,h}$ can be estimated (w.h.p.) given $\hat{\vx}_v, \ve_h, \rho, \vx_v$. 
But even with just the noisy data and not $\vx_v$, we can estimate $\delta_{v,h}$ w.h.p. 
by extrapolating the correlation of the bits of $\hat{\vx}_v$ that are left as $+1$ (a $1-\rho$ fraction) with the corresponding values in $\ve_h$. 


\section{Experiments}
\label{sec:experiments}

In this section we compare our approach 
empirically to standard autoencoders with one hidden layer (termed AE here) trained with backpropagation. 
Our goal is simply to verify that our very distinct approach is competitive in reconstruction performance with cross-entropy loss.

The datasets we use are first normalized to $[0,1]$, and then binarized by sampling each pixel stochastically in proportion to its intensity, 
following prior work (\cite{SM08}). 
Choosing between binary and real-valued encodings in $\algname$ requires just a line of code, 
to project the encodings into $[-1,1]^V$ after convex optimization updates to compute $\encode (\cdot)$. 
We use Adagrad (\cite{DHS11}) for the convex minimizations of our algorithms; 
we observed that their performance is not very sensitive to the choice of optimization method, explained by our approach's convexity. 

We compare to a basic single-layer AE trained with the Adam method with default parameters in \cite{KB14}. 
Other models like variational autoencoders (\cite{KW13}) are not shown here because they do not aim to optimize reconstruction loss 
or are not comparably general autoencoding architectures (see Appendix \ref{sec:expdetails}). 
We try 32 and 100 hidden units for both algorithms, and try both binary and unconstrained real-valued encodings; 
the respective AE uses logistic and ReLU transfer functions for the encoding neurons. The results are in Table \ref{tab:recloss}. 

The reconstruction performance of $\algname$ indicates that it can encode information very well using pairwise correlations. 
Loss can become extremely low when $H$ is raised, giving $\vB$ the capacity to encode far more information. 
The performance is marginally better with binary hidden units than unconstrained ones, in accordance with the spirit of our derivations.

\begin{table}[t]
\caption{Cross-entropy reconstruction losses for $\algname$ and a vanilla single-layer autoencoder, 
with binary and unconstrained real-valued encodings.}
\label{tab:recloss}
\begin{center}
\begin{tabular}{  c || c | c || c | c | }
 & $\algname$ (bin.) & $\algname$ (real) & AE (bin.) & AE (real) \\ \hline \hline
MNIST, $H = 32$ &  51.9 & 53.8 & 65.2 & 64.3 \\ 
MNIST, $H = 100$ & 9.2  & 9.9 & 26.8 & 25.0 \\ 
Omniglot, $H = 32$ & 76.1 & 77.2  & 93.1 & 90.6 \\ 
Omniglot, $H = 100$ & 12.1 & 13.2 & 46.6 & 45.4 \\ 
Caltech-101, $H = 32$ & 54.5  & 54.9  & 97.5 & 87.6 \\ 
Caltech-101, $H = 100$ & 7.1  & 7.1  & 64.3 & 45.4 \\ 
notMNIST, $H = 32$ & 121.9  & 122.4  & 149.6 & 141.8 \\ 
notMNIST, $H = 100$ & 62.2  & 63.0  & 99.6 & 92.1 \\ 
\end{tabular}
\end{center}
\end{table}

We also try learning just the decoding layer of Sec. \ref{sec:optdecode}, on the encoded representation of the AE. 
This is motivated by the fact that  establishes our decoding method to be worst-case optimal given any $\vE$ and $\vB$. 
We find the results to be significantly worse than the AE alone on all datasets 
(reconstruction loss of $\sim 171/133$ on MNIST, and $\sim 211/134$ on Omniglot, with $32/100$ hidden units respectively). 
This reflects the AE's backprop training propagating information about the data beyond pairwise correlations 
through non-convex function compositions -- however, the cost of this is that they are more difficult to optimize. 
The representations learned by the $\encode$ function of $\algname$ are quite different and capture much more of the pairwise correlation information, 
which is used by the decoding layer in a worst-case optimal fashion. 
We attempt to visually depict the differences between the representations in Fig. \ref{fig:repdifferences}. 

\begin{figure}[h]
\begin{center}
\includegraphics[width=0.99\linewidth]{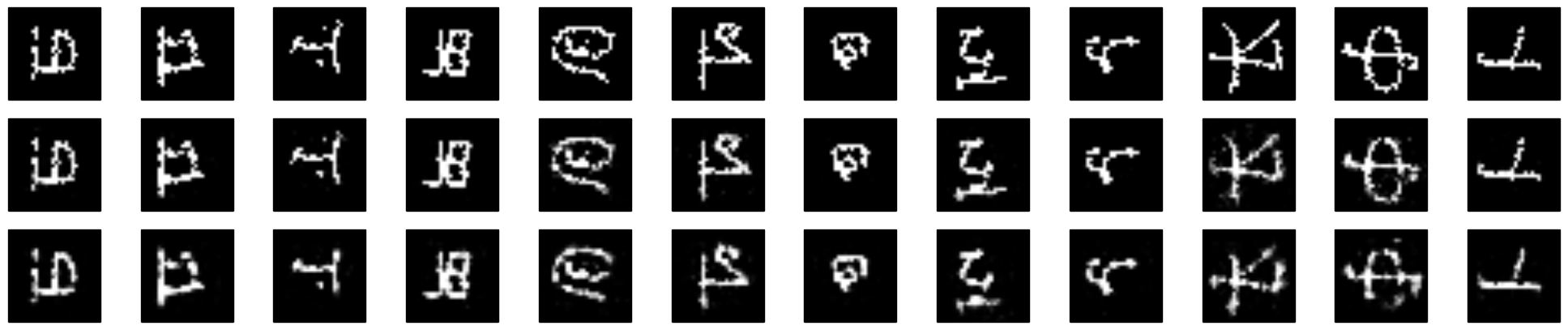}
\end{center}
\caption{Top row: random test images from Omniglot. 
Middle and bottom rows: reconstructions of $\algname$ and AE with $H=100$ binary hidden units. 
Difference in quality is particularly noticeable in the 1st, 5th, 8th, and 11th columns.}
\label{fig:oglot100}
\end{figure}

\begin{figure}[h]
\begin{center}
\includegraphics[width=0.99\linewidth]{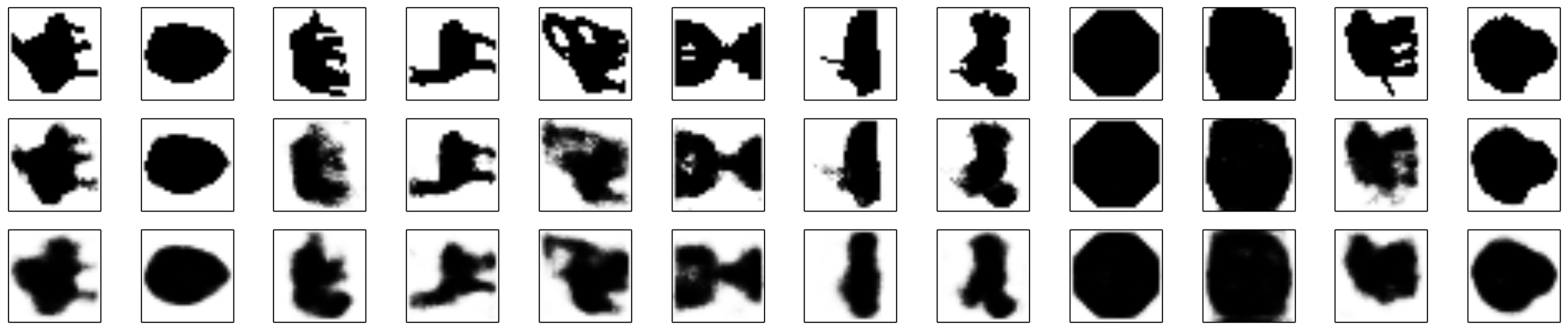}
\end{center}
\caption{As Fig. \ref{fig:oglot100}, with $H=32$ on Caltech-101 silhouettes.}
\label{fig:silh32}
\end{figure}

\begin{figure}[h]
\begin{center}
\includegraphics[width=0.99\linewidth]{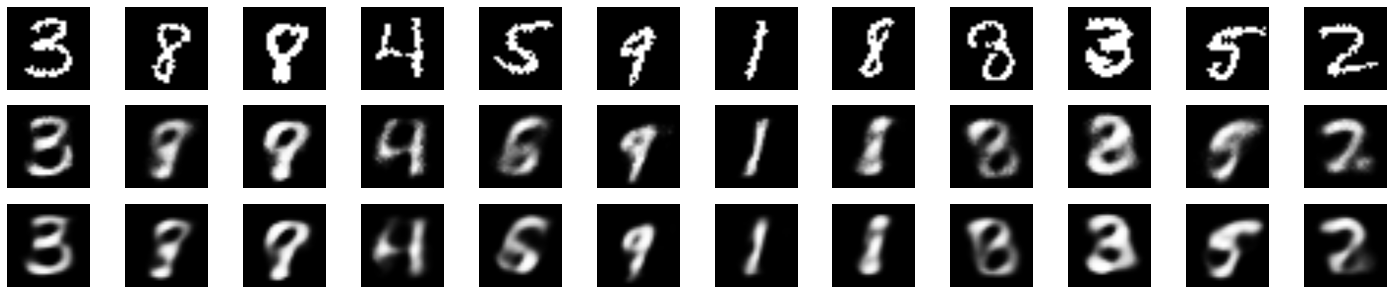} 
\includegraphics[width=0.99\linewidth]{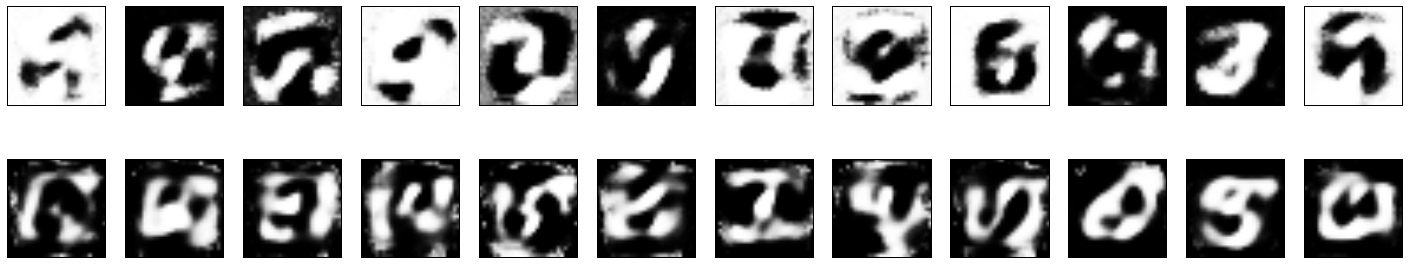}
\end{center}
\caption{Top three rows: the reconstructions of random test images from MNIST ($H=12$), as in Fig. \ref{fig:oglot100}. 
$\algname$ achieves loss $105.1$ here, and AE $111.2$.
Fourth and fifth rows: visualizations of all the hidden units of $\algname$ and AE, respectively. 
It is not possible to visualize the $\algname$ encoding units by the image that maximally activates them, as commonly done, 
because of the form of the $\encode$ function which depends on $\vW$ and lacks explicit encoding weights. 
So each hidden unit $h$ is depicted by the visible decoding of the encoded representation which has bit $h$ "on" 
and all other bits "off." 
(If this were PCA with a linear decoding layer, this would simply represent hidden unit $h$ by its corresponding principal component vector, 
the decoding of the $h^{th}$ canonical basis vector in $\RR^H$.)}
\label{fig:repdifferences}
\end{figure}

As discussed in Sec. \ref{sec:relworkdisc}, we do not claim that our method will always achieve the best empirical reconstruction loss, 
even among single-layer autoencoders. 
We would like to make the encoding function quicker to compute, as well.
But we believe this paper's results, especially when $H$ is high, illustrate the potential of using pairwise correlations for autoencoding 
as in our approach, 
learning to encode with alternating convex minimization and extremely strong worst-case robustness guarantees.

\subsubsection*{Acknowledgments}
I am grateful to Jack Berkowitz, Sanjoy Dasgupta, and Yoav Freund for helpful discussions; 
Daniel Hsu and Akshay Krishnamurthy for instructive examples; 
and Gary Cottrell for enjoyable chats.

\bibliography{AE_short}

\begin{thebibliography}{41}
\providecommand{\natexlab}[1]{#1}
\providecommand{\url}[1]{\texttt{#1}}
\expandafter\ifx\csname urlstyle\endcsname\relax
  \providecommand{\doi}[1]{doi: #1}\else
  \providecommand{\doi}{doi: \begingroup \urlstyle{rm}\Url}\fi

\bibitem[Ackley et~al.(1985)Ackley, Hinton, and Sejnowski]{AHS85}
David~H Ackley, Geoffrey~E Hinton, and Terrence~J Sejnowski.
\newblock A learning algorithm for boltzmann machines.
\newblock \emph{Cognitive science}, 9\penalty0 (1):\penalty0 147--169, 1985.

\bibitem[Arora et~al.(2014)Arora, Bhaskara, Ge, and Ma]{ABGM14}
Sanjeev Arora, Aditya Bhaskara, Rong Ge, and Tengyu Ma.
\newblock Provable bounds for learning some deep representations.
\newblock In \emph{Proceedings of the 31st International Conference on Machine
  Learning (ICML-14)}, pp.\  584--592, 2014.

\bibitem[Aslan et~al.(2014)Aslan, Zhang, and Schuurmans]{AZS14}
{\"O}zlem Aslan, Xinhua Zhang, and Dale Schuurmans.
\newblock Convex deep learning via normalized kernels.
\newblock In \emph{Advances in Neural Information Processing Systems}, pp.\
  3275--3283, 2014.

\bibitem[Bach(2014)]{Bach14}
Francis Bach.
\newblock Breaking the curse of dimensionality with convex neural networks.
\newblock \emph{arXiv preprint arXiv:1412.8690}, 2014.

\bibitem[Baldi(2012)]{Baldi2012}
Pierre Baldi.
\newblock Autoencoders, unsupervised learning, and deep architectures.
\newblock \emph{Unsupervised and Transfer Learning Challenges in Machine
  Learning, Volume 7}, pp.\ ~43, 2012.

\bibitem[Baldi \& Hornik(1989)Baldi and Hornik]{BH89}
Pierre Baldi and Kurt Hornik.
\newblock Neural networks and principal component analysis: Learning from
  examples without local minima.
\newblock \emph{Neural networks}, 2\penalty0 (1):\penalty0 53--58, 1989.

\bibitem[Balsubramani \& Freund(2015{\natexlab{a}})Balsubramani and
  Freund]{BF15}
Akshay Balsubramani and Yoav Freund.
\newblock Optimally combining classifiers using unlabeled data.
\newblock In \emph{Conference on Learning Theory (COLT)}, 2015{\natexlab{a}}.

\bibitem[Balsubramani \& Freund(2015{\natexlab{b}})Balsubramani and
  Freund]{BF15b}
Akshay Balsubramani and Yoav Freund.
\newblock Scalable semi-supervised classifier aggregation.
\newblock In \emph{Advances in Neural Information Processing Systems (NIPS)},
  2015{\natexlab{b}}.

\bibitem[Balsubramani \& Freund(2016)Balsubramani and Freund]{BF16}
Akshay Balsubramani and Yoav Freund.
\newblock Optimal binary classifier aggregation for general losses.
\newblock In \emph{Advances in Neural Information Processing Systems (NIPS)},
  2016.
\newblock arXiv:1510.00452.

\bibitem[Balsubramani et~al.(2013)Balsubramani, Dasgupta, and Freund]{BDF13}
Akshay Balsubramani, Sanjoy Dasgupta, and Yoav Freund.
\newblock The fast convergence of incremental pca.
\newblock In \emph{Advances in Neural Information Processing Systems (NIPS)},
  pp.\  3174--3182, 2013.

\bibitem[Bartlett(1998)]{B98}
Peter~L Bartlett.
\newblock The sample complexity of pattern classification with neural networks:
  the size of the weights is more important than the size of the network.
\newblock \emph{IEEE Transactions on Information Theory}, 44\penalty0
  (2):\penalty0 525--536, 1998.

\bibitem[Bengio(2009)]{Ben09}
Yoshua Bengio.
\newblock Learning deep architectures for ai.
\newblock \emph{Foundations and Trends in Machine Learning}, 2\penalty0
  (1):\penalty0 1--127, 2009.

\bibitem[Bengio et~al.(2005)Bengio, Roux, Vincent, Delalleau, and
  Marcotte]{BRVDM05}
Yoshua Bengio, Nicolas~L Roux, Pascal Vincent, Olivier Delalleau, and Patrice
  Marcotte.
\newblock Convex neural networks.
\newblock In \emph{Advances in neural information processing systems (NIPS)},
  pp.\  123--130, 2005.

\bibitem[Bengio et~al.(2013{\natexlab{a}})Bengio, Courville, and
  Vincent]{BCV13}
Yoshua Bengio, Aaron Courville, and Pierre Vincent.
\newblock Representation learning: A review and new perspectives.
\newblock \emph{Pattern Analysis and Machine Intelligence, IEEE Transactions
  on}, 35\penalty0 (8):\penalty0 1798--1828, 2013{\natexlab{a}}.

\bibitem[Bengio et~al.(2013{\natexlab{b}})Bengio, Yao, Alain, and
  Vincent]{BYAV13}
Yoshua Bengio, Li~Yao, Guillaume Alain, and Pascal Vincent.
\newblock Generalized denoising auto-encoders as generative models.
\newblock In \emph{Advances in Neural Information Processing Systems (NIPS)},
  pp.\  899--907, 2013{\natexlab{b}}.

\bibitem[Bourlard \& Kamp(1988)Bourlard and Kamp]{BK88}
Herv{\'e} Bourlard and Yves Kamp.
\newblock Auto-association by multilayer perceptrons and singular value
  decomposition.
\newblock \emph{Biological cybernetics}, 59\penalty0 (4-5):\penalty0 291--294,
  1988.

\bibitem[Burda et~al.(2016)Burda, Grosse, and Salakhutdinov]{BGS16}
Yuri Burda, Roger Grosse, and Ruslan Salakhutdinov.
\newblock Importance weighted autoencoders.
\newblock \emph{International Conference on Learning Representations (ICLR)},
  2016.
\newblock arXiv preprint arXiv:1509.00519.

\bibitem[Cesa-Bianchi \& Lugosi(2006)Cesa-Bianchi and Lugosi]{CBL06}
Nicolo Cesa-Bianchi and G{\`a}bor Lugosi.
\newblock \emph{Prediction, Learning, and Games}.
\newblock Cambridge University Press, New York, NY, USA, 2006.

\bibitem[Dauphin et~al.(2014)Dauphin, Pascanu, Gulcehre, Cho, Ganguli, and
  Bengio]{DPGCGB14}
Yann~N Dauphin, Razvan Pascanu, Caglar Gulcehre, Kyunghyun Cho, Surya Ganguli,
  and Yoshua Bengio.
\newblock Identifying and attacking the saddle point problem in
  high-dimensional non-convex optimization.
\newblock In \emph{Advances in neural information processing systems (NIPS)},
  pp.\  2933--2941, 2014.

\bibitem[Duchi et~al.(2011)Duchi, Hazan, and Singer]{DHS11}
John Duchi, Elad Hazan, and Yoram Singer.
\newblock Adaptive subgradient methods for online learning and stochastic
  optimization.
\newblock \emph{The Journal of Machine Learning Research}, 12:\penalty0
  2121--2159, 2011.

\bibitem[Freund \& Haussler(1992)Freund and Haussler]{FH92}
Yoav Freund and David Haussler.
\newblock Unsupervised learning of distributions on binary vectors using two
  layer networks.
\newblock In \emph{Advances in Neural Information Processing Systems (NIPS)},
  pp.\  912--919, 1992.

\bibitem[Goodfellow et~al.(2014)Goodfellow, Pouget-Abadie, Mirza, Xu,
  Warde-Farley, Ozair, Courville, and Bengio]{GPM+14}
Ian Goodfellow, Jean Pouget-Abadie, Mehdi Mirza, Bing Xu, David Warde-Farley,
  Sherjil Ozair, Aaron Courville, and Yoshua Bengio.
\newblock Generative adversarial nets.
\newblock In \emph{Advances in Neural Information Processing Systems (NIPS)},
  pp.\  2672--2680, 2014.

\bibitem[Gorski et~al.(2007)Gorski, Pfeuffer, and Klamroth]{GPK07}
Jochen Gorski, Frank Pfeuffer, and Kathrin Klamroth.
\newblock Biconvex sets and optimization with biconvex functions: a survey and
  extensions.
\newblock \emph{Mathematical Methods of Operations Research}, 66\penalty0
  (3):\penalty0 373--407, 2007.

\bibitem[He et~al.(2015)He, Zhang, Ren, and Sun]{HZRS15}
Kaiming He, Xiangyu Zhang, Shaoqing Ren, and Jian Sun.
\newblock Deep residual learning for image recognition.
\newblock \emph{arXiv preprint arXiv:1512.03385}, 2015.

\bibitem[Hinton et~al.(1995)Hinton, Dayan, Frey, and Neal]{HDFN95}
Geoffrey~E Hinton, Peter Dayan, Brendan~J Frey, and Radford~M Neal.
\newblock The" wake-sleep" algorithm for unsupervised neural networks.
\newblock \emph{Science}, 268\penalty0 (5214):\penalty0 1158--1161, 1995.

\bibitem[Janzamin et~al.(2015)Janzamin, Sedghi, and Anandkumar]{JSA15}
Majid Janzamin, Hanie Sedghi, and Anima Anandkumar.
\newblock Beating the perils of non-convexity: Guaranteed training of neural
  networks using tensor methods.
\newblock \emph{arXiv preprint arXiv:1506.08473}, 2015.

\bibitem[Jordan(1995)]{J95}
Michael~I Jordan.
\newblock Why the logistic function? a tutorial discussion on probabilities and
  neural networks, 1995.

\bibitem[Kingma \& Ba(2014)Kingma and Ba]{KB14}
Diederik Kingma and Jimmy Ba.
\newblock Adam: A method for stochastic optimization.
\newblock \emph{arXiv preprint arXiv:1412.6980}, 2014.

\bibitem[Kingma \& Welling(2013)Kingma and Welling]{KW13}
Diederik~P Kingma and Max Welling.
\newblock Auto-encoding variational bayes.
\newblock \emph{arXiv preprint arXiv:1312.6114}, 2013.

\bibitem[LeCun et~al.(2015)LeCun, Bengio, and Hinton]{LBH15}
Yann LeCun, Yoshua Bengio, and Geoffrey Hinton.
\newblock Deep learning.
\newblock \emph{Nature}, 521\penalty0 (7553):\penalty0 436--444, 2015.

\bibitem[Li et~al.(2015)Li, Swersky, and Zemel]{LSZ15}
Yujia Li, Kevin Swersky, and Rich Zemel.
\newblock Generative moment matching networks.
\newblock In \emph{Proceedings of the 32nd International Conference on Machine
  Learning (ICML-15)}, pp.\  1718--1727, 2015.

\bibitem[Livni et~al.(2014)Livni, Shalev-Shwartz, and Shamir]{LSS14}
Roi Livni, Shai Shalev-Shwartz, and Ohad Shamir.
\newblock On the computational efficiency of training neural networks.
\newblock In \emph{Advances in Neural Information Processing Systems (NIPS)},
  pp.\  855--863, 2014.

\bibitem[Makhzani et~al.(2015)Makhzani, Shlens, Jaitly, and
  Goodfellow]{MSJGF15}
Alireza Makhzani, Jonathon Shlens, Navdeep Jaitly, and Ian Goodfellow.
\newblock Adversarial autoencoders.
\newblock \emph{arXiv preprint arXiv:1511.05644}, 2015.

\bibitem[Rasmus et~al.(2015)Rasmus, Berglund, Honkala, Valpola, and
  Raiko]{RBHVR15}
Antti Rasmus, Mathias Berglund, Mikko Honkala, Harri Valpola, and Tapani Raiko.
\newblock Semi-supervised learning with ladder networks.
\newblock In \emph{Advances in Neural Information Processing Systems}, pp.\
  3546--3554, 2015.

\bibitem[Rumelhart \& McClelland(1987)Rumelhart and McClelland]{RM87}
David~E Rumelhart and James~L McClelland.
\newblock Parallel distributed processing, explorations in the microstructure
  of cognition. vol. 1: Foundations.
\newblock \emph{Computational Models of Cognition and Perception, Cambridge:
  MIT Press}, 1987.

\bibitem[Salakhutdinov \& Murray(2008)Salakhutdinov and Murray]{SM08}
Ruslan Salakhutdinov and Iain Murray.
\newblock On the quantitative analysis of deep belief networks.
\newblock In \emph{Proceedings of the 25th International Conference on Machine
  Learning (ICML)}, pp.\  872--879, 2008.

\bibitem[Shamir(2016)]{S16}
Ohad Shamir.
\newblock Convergence of stochastic gradient descent for pca.
\newblock \emph{International Conference on Machine Learning (ICML)}, 2016.
\newblock arXiv preprint arXiv:1509.09002.

\bibitem[Smolensky(1986)]{Smo86}
P~Smolensky.
\newblock Information processing in dynamical systems: foundations of harmony
  theory.
\newblock In \emph{Parallel distributed processing: explorations in the
  microstructure of cognition, vol. 1}, pp.\  194--281. MIT Press, 1986.

\bibitem[Vincent et~al.(2008)Vincent, Larochelle, Bengio, and Manzagol]{VLBM08}
Pascal Vincent, Hugo Larochelle, Yoshua Bengio, and Pierre-Antoine Manzagol.
\newblock Extracting and composing robust features with denoising autoencoders.
\newblock In \emph{Proceedings of the 25th international conference on Machine
  learning (ICML)}, pp.\  1096--1103. ACM, 2008.

\bibitem[Vincent et~al.(2010)Vincent, Larochelle, Lajoie, Bengio, and
  Manzagol]{VLLBM10}
Pascal Vincent, Hugo Larochelle, Isabelle Lajoie, Yoshua Bengio, and
  Pierre-Antoine Manzagol.
\newblock Stacked denoising autoencoders: Learning useful representations in a
  deep network with a local denoising criterion.
\newblock \emph{The Journal of Machine Learning Research}, 11:\penalty0
  3371--3408, 2010.

\bibitem[Zhang et~al.(2016)Zhang, Liang, and Wainwright]{ZLW16}
Yuchen Zhang, Percy Liang, and Martin~J Wainwright.
\newblock Convexified convolutional neural networks.
\newblock \emph{arXiv preprint arXiv:1609.01000}, 2016.

\end{thebibliography}
\bibliographystyle{iclr2017_conference}
\appendix
\newpage

\section{Experimental Details}
\label{sec:expdetails}

In addition to MNIST, we used the preprocessed version of the Omniglot dataset in \cite{BGS16}, 
split 1 of the Caltech-101 Silhouettes dataset, 
and the small notMNIST dataset. 
Only notMNIST comes without a predefined split, so the displayed results use 10-fold cross-validation.
Non-binarized versions of all datasets resulted in nearly identical $\algname$ performance (not shown), 
as would be expected from its derivation using expected pairwise correlations. 

We used minibatches of size 250. 
All autoencoders were initialized with the 'Xavier' initialization and trained for 500 epochs or using early stopping on the test set. 

We did not evaluate against other types of autoencoders which regularize (\cite{KW13}) or are otherwise not trained for direct reconstruction loss minimization. 
Also, not shown is the performance of a standard convolutional autoencoder (32-bit representation, depth-3 64-64-32 (en/de)coder) which is somewhat better than the standard autoencoder, 
but is still outperformed by $\algname$ on our datasets. 
A deeper architecture could quite possibly achieve superior performance, but the greater number of channels through which information is propagated 
makes fair comparison with our flat fully-connected approach difficult. 
We consider extension of our $\algname$ approach to such architectures to be fascinating future work.

\subsection{Further Results}

Our bound on worst-case loss is invariably quite tight, as shown in Fig. \ref{fig:slacktracksmmxloss}. 
Similar results are found on all datasets. 
This is consistent with our conclusions about the nature of the $\algname$ representations -- conveying almost exactly the information available in pairwise correlations.

\begin{figure}[h]
\begin{center}
\includegraphics[width=0.6\linewidth]{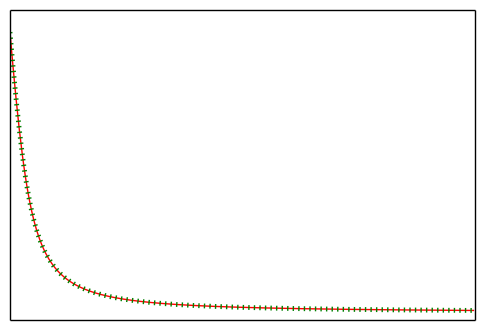}
\end{center}
\caption{Actual reconstruction loss to real data (red) and slack function [objective function] value (dotted green), 
during an Adagrad optimization to learn $\vW$ using the optimal $\vE, \vB$. 
Monotonicity is expected since this is a convex optimization. 
The objective function value theoretically upper-bounds the actual loss, and practically tracks it nearly perfectly.}
\label{fig:slacktracksmmxloss}
\end{figure}

A 2D visualization of MNIST in Fig. \ref{fig:mnist_vis}, 
showing that even with just two hidden units there is enough information in pairwise correlations 
for $\algname$ to learn a sensible embedding. 
We also include more pictures of our autoencoders' reconstructions, 
and visualizations of the hidden units when $H=100$ in Fig. \ref{fig:oglot100enc_vis}.

\begin{figure}[h]
\begin{center}
\includegraphics[width=0.99\linewidth]{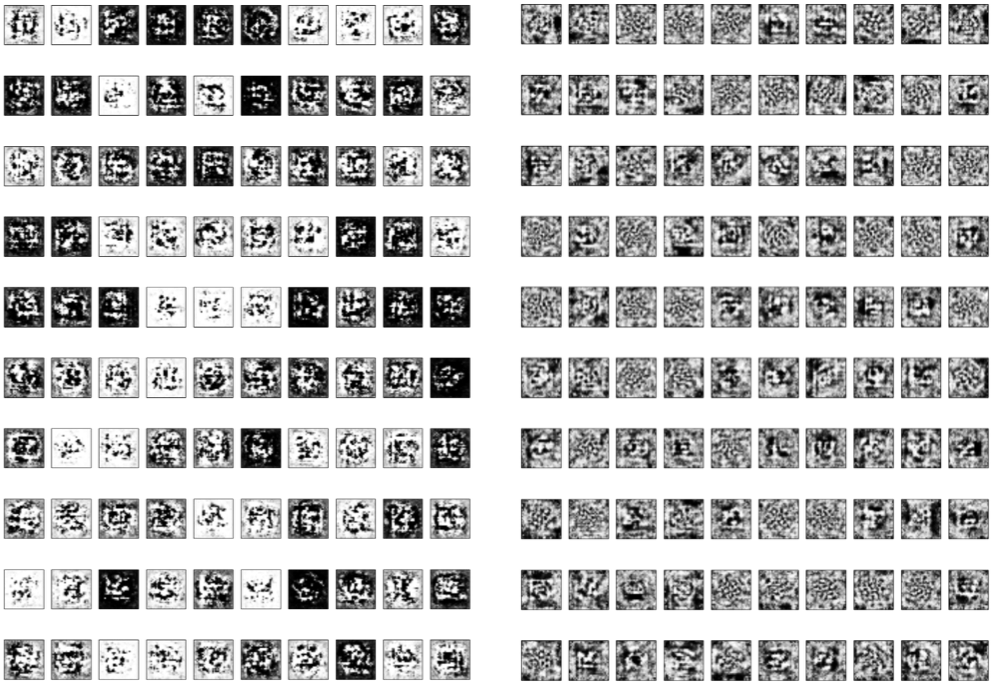}
\end{center}
\caption{Visualizations of all the hidden units of PC-AE (left) and AE (right) from Omniglot for $H=100$, as in Fig. \ref{fig:repdifferences}.}
\label{fig:oglot100enc_vis}
\end{figure}

\begin{figure}[h]
\begin{center}
\includegraphics[height=1.9in, width=0.9\linewidth]{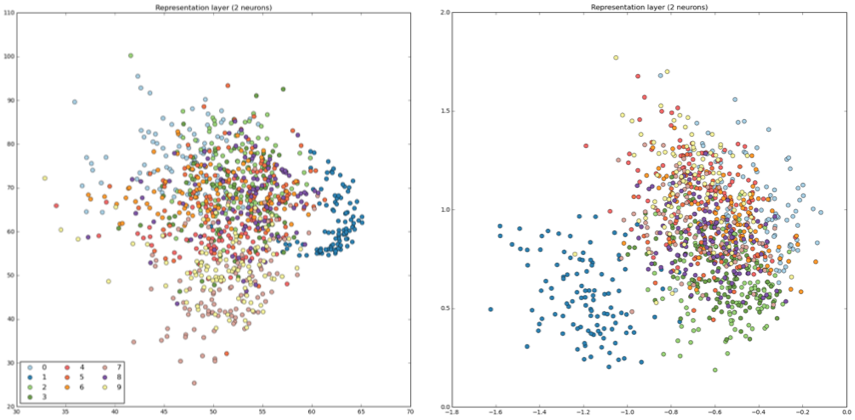}
\end{center}
\caption{AE (left) and $\algname$ (right) visualizations of a random subset of MNIST test data, with $H=2$ real-valued hidden units, and colors corresponding to class labels (legend at left). $\algname$'s loss is $\sim 189$ here, and that of AE is $\sim 179$.}
\label{fig:mnist_vis}
\end{figure}

\begin{figure}[h]
\begin{center}
\includegraphics[width=0.99\linewidth]{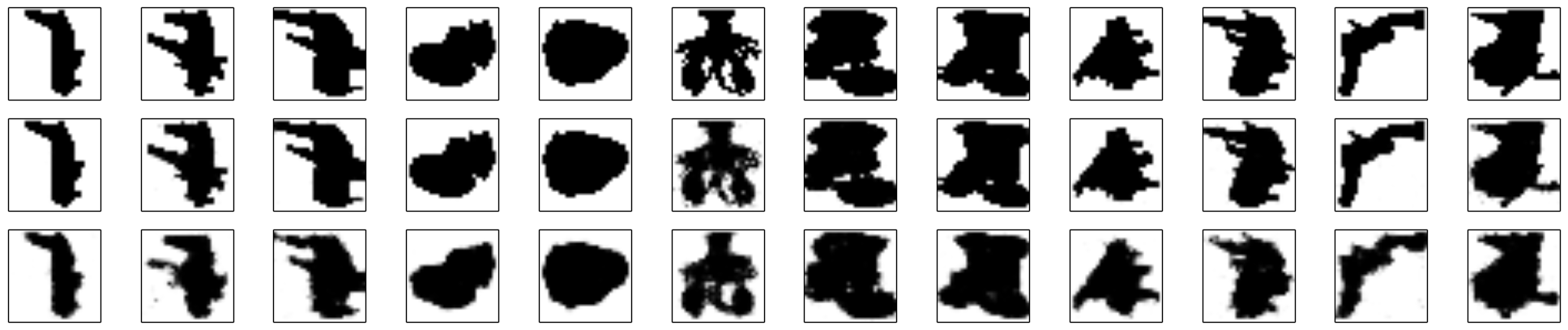}
\end{center}
\caption{As Fig. \ref{fig:oglot100}, with $H=100$ on Caltech-101 silhouettes.}
\label{fig:silh100}
\end{figure}

\begin{figure}[h]
\begin{center}
\includegraphics[width=0.99\linewidth]{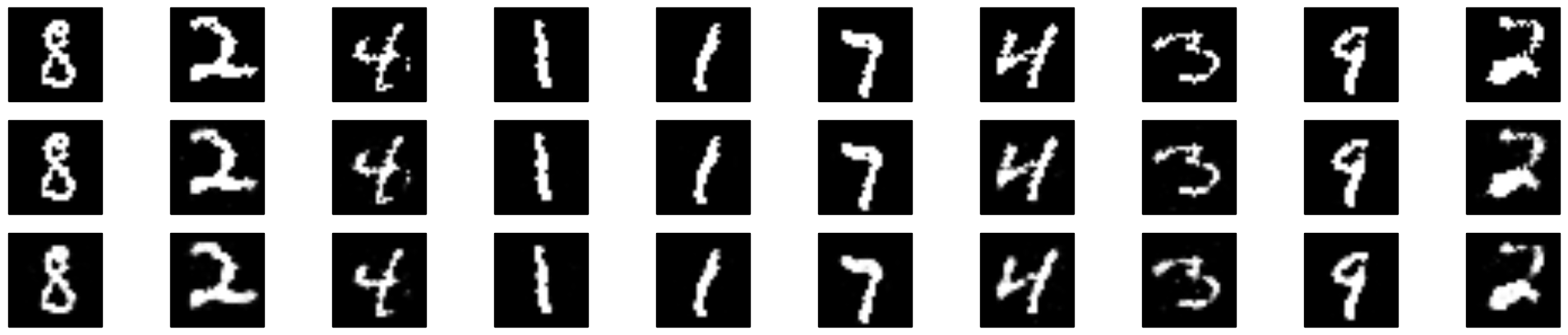}
\end{center}
\caption{As Fig. \ref{fig:oglot100}, with $H=100$ on MNIST.}
\label{fig:mnist100}
\end{figure}

\begin{figure}[h]
\begin{center}
\includegraphics[width=0.99\linewidth]{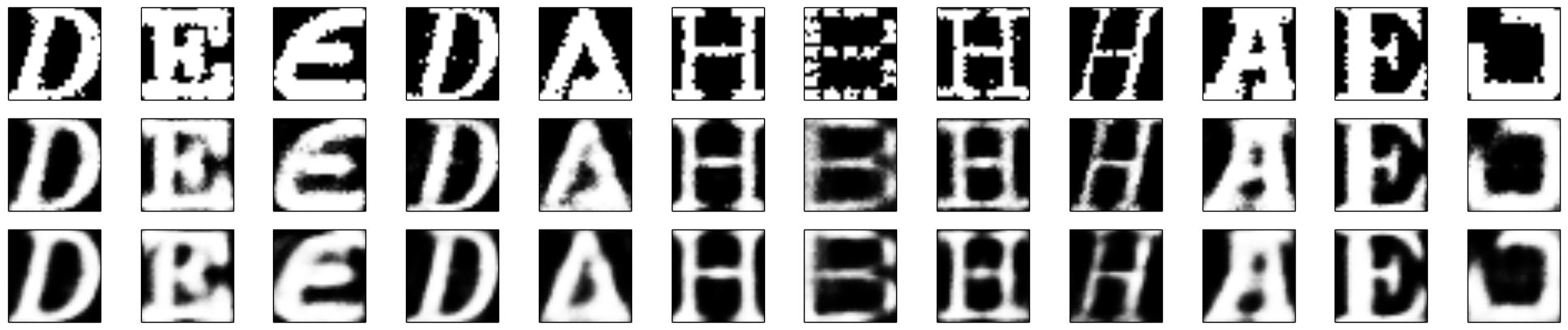}
\end{center}
\caption{As Fig. \ref{fig:oglot100}, with $H=32$ on notMNIST.}
\label{fig:nmn32}
\end{figure}

\subsection{$\algname$ Interpretation and Implementation Details}
\label{sec:details}

Here we give some details that are useful for interpretation and implementation of the proposed method.

\subsubsection{Encoding}

\Cref{prop:encfn} defines the encoding function for any data example $\vx$ as the vector that minimizes the total feature distortion, 
summed over the bits in the decoding, rewritten here for convenience: 
\begin{align}
\label{eq:multidencode}
\encode ( \vx^{(i)} ; \vW) := 
\argmin_{\ve \in [-1,1]^{H}} \; \sum_{v=1}^V \; \lrb{ - x_v^{(i)} \vw_v^{ \top} \ve^{(i)} + \Psi (\vw_v^{ \top} \ve^{(i)} ) }
\end{align}

Doing this on multiple examples at once (in memory as a minibatch) can be much faster than on each example separately. 
We can now compute the gradient of the objective function w.r.t. each example $i \in [n]$,  
writing the gradient w.r.t. example $i$ as column $i$ of a matrix $\vG \in \RR^{H \times n}$. 
$\vG$ can be calculated efficiently in a number of ways, for example as follows:
\begin{itemize}
\item
Compute matrix of \textbf{hallucinated data} $\breve{\vX} := \Psi' (\vW \vE) \in \RR^{V \times n}$. 
\item
Subtract $\vX$ to compute \textbf{residuals} $\vR := \breve{\vX} - \vX \in \RR^{V \times n}$.
\item
Compute $\vG = \frac{1}{n} \vW^\top \vR \in \RR^{H \times n}$. 
\end{itemize}

Optimization then proceeds with gradient descent using $\vG$, with the step size found using line search. 
Note that since the objective function is convex, the optimum $\vE^*$ leads to optimal residuals $\vR^* \in \RR^{V \times n}$ 
such that $\vG = \frac{1}{n} \vW^\top \vR^* = \vzero^{H \times n}$, 
so each column of $\vR^*$ is in the null space of $\vW^\top$, which maps the residual vectors to the encoded space. 
We conclude that although the compression is not perfect (so the optimal residuals $\vR^* \neq \vzero^{V \times n}$ in general), 
each column of $\vR^*$ is orthogonal to the decoding weights at an equilibrium towards which 
the convex minimization problem of \eqref{eq:multidencode} is guaranteed to stably converge.

\subsubsection{Decoding}

The decoding step finds $\vW$ to ensure accurate decoding of the given encodings $\vE$ with correlations $\vB$, 
solving the convex minimization problem: 
\begin{align}
\label{eq:decodeatonce}
\vW^* = \argmin_{ \vW \in \RR^{V \times H} } \; \sum_{v=1}^V \lrb{ - \vb_v^\top \vw_v + \frac{1}{n} \sum_{i=1}^n \Psi (\vw_v^\top \ve^{(i)} ) }
\end{align}
This can be minimized by first-order convex optimization. 
The gradient of \eqref{eq:decodeatonce} at $\vW$ is: 
\begin{align}
\label{eq:gradfdistort}
- \vB + \frac{1}{n} [\Psi' (\vW \vE )] \vE^\top
\end{align}
The second term can be understood as ``hallucinated" pairwise correlations $\breve{\vB}$, 
between bits of the encoded examples $\vE$ and bits of their decodings under the current weights, $\breve{\vX} := \Psi' (\vW \vE )$. 
The hallucinated correlations can be written as $\breve{\vB} := \frac{1}{n} \breve{\vX} \vE^\top$. 
Therefore, \eqref{eq:gradfdistort} can be interpreted as the residual correlations $\breve{\vB} - \vB$. 
Since the slack function of \eqref{eq:decodeatonce} is convex, 
the optimum $\vW^*$ leads to hallucinated correlations $\breve{\vB}^* = \vB$, 
which is the limit reached by the optimization algorithm after many iterations.

\section{Allowing Randomized Data and Encodings}
\label{sec:datarandomized}

In this paper, we represent the bit-vector data in a randomized way in $[-1,1]^{V}$. 
Randomizing the data only relaxes the constraints on the adversary in the game we play; 
so at worst we are working with an upper bound on worst-case loss, instead of the exact minimax loss itself, 
erring on the conservative side. 
Here we briefly justify the bound as being essentially tight, which we also see empirically in this paper's experiments. 

In the formulation of \Cref{sec:optdecode}, the only information we have about the data is its pairwise correlations with the encoding units. 
When the data are abundant ($n$ large), then w.h.p. these correlations are close to their expected values 
over the data's internal randomization, so representing them as continuous values w.h.p. 
results in the same $\vB$ and therefore the same solutions for $\vE, \vW$. 
We are effectively allowing the adversary to play each bit's conditional probability of firing, rather than the binary realization of that probability. 

This allows us to apply minimax theory and duality to considerably simplify the problem to a convex optimization, 
when it would otherwise be nonconvex, and computationally hard (\cite{Baldi2012}). 
The fact that we are only using information about the data through its \emph{expected} pairwise correlations makes this possible.

The above also applies to the encodings and their internal randomization, 
allowing us to learn binary randomized encodings by projecting to the convex set $[-1,1]^{H}$.

\section{Proofs}
\label{sec:proofs}


\begin{proof}[Proof of Theorem \ref{thm:decfn}]
Writing $\Gamma (\tilde{x}_v^{(i)}) := \ell_{-} (\tilde{x}_v^{(i)}) - \ell_{+} (\tilde{x}_v^{(i)}) = \ln \lrp{\frac{1 + \tilde{x}_v^{(i)}}{1 - \tilde{x}_v^{(i)} }}$ for convenience, 
we can simplify $\cL^*$, using the definition of the loss \eqref{eq:defxentloss}, and Lagrange duality for all $VH$ constraints involving $\vB$. 

This leads to the following chain of equalities, where for brevity the constraint sets are sometimes omitted when clear, 
and we write $\vX$ as shorthand for the data $\vx^{(1)}, \dots, \vx^{(n)}$ and $\tilde{\vX}$ analogously for the reconstructions.
\begin{align}
\label{eq:defofdecloss}
\cL^* &= 
\frac{1}{2} \min_{\substack{\vtildex^{(1)}, \dots, \vtildex^{(n)} \\ \in [-1,1]^V}} \; \max_{\substack{ \vx^{(1)}, \dots, \vx^{(n)} \in [-1,1]^V , \\ \forall v \in [V] : \;\frac{1}{n} \vE \vx_v = \vb_v }} 
\; \frac{1}{n} \sum_{i=1}^n \sum_{v=1}^V \lrb{ \lrp{ 1 + x_v^{(i)} } \ell_{+} (\tilde{x}_v^{(i)}) + \lrp{ 1 - x_v^{(i)} } \ell_{-} (\tilde{x}_v^{(i)}) }  
\nonumber \\
&= \frac{1}{2} \min_{\tilde{\vX} } \; \max_{ \vX } \; \min_{ \vW \in \RR^{V \times H}} 
\; \lrb{ \frac{1}{n} \sum_{i=1}^n \sum_{v=1}^V \lrp{ \ell_{+} (\tilde{x}_v^{(i)}) + \ell_{-} (\tilde{x}_v^{(i)}) - x_v^{(i)} \Gamma (\tilde{x}_v^{(i)}) } + \sum_{v=1}^V \vw_v^\top \lrp{ \frac{1}{n} \vE \vx_v - \vb_v } }
\nonumber \\
&\stackrel{(a)}{=} \frac{1}{2} \min_{ \vw_1, \dots, \vw_V } \; \lrb{ - \sum_{v=1}^V \vb_v^\top \vw_v + \frac{1}{n} \min_{\tilde{\vX} } \; \max_{ \vX } \; 
\sum_{v=1}^V \lrb{ \sum_{i=1}^n \lrp{ \ell_{+} (\tilde{x}_v^{(i)}) + \ell_{-} (\tilde{x}_v^{(i)}) - x_v^{(i)} \Gamma (\tilde{x}_v^{(i)}) } + \vw_v^\top \vE \vx_v } } 
\nonumber \\
&= \frac{1}{2} \min_{ \vw_1, \dots, \vw_V } \; \Bigg[ - \sum_{v=1}^V \vb_v^\top \vw_v + \frac{1}{n} \min_{\tilde{\vX} } \; \sum_{i=1}^n 
\sum_{v=1}^V \lrb{ \ell_{+} (\tilde{x}_v^{(i)}) + \ell_{-} (\tilde{x}_v^{(i)}) + 
\max_{ \vx^{(i)}  \in [-1,1]^V } \; x_v^{(i)} \lrp{ \vw_v^\top \ve^{(i)} - \Gamma (\tilde{x}_v^{(i)}) }} \Bigg]
\end{align}
where $(a)$ uses the minimax theorem (\cite{CBL06}), 
which can be applied as in linear programming, because the objective function is linear in $\vx^{(i)}$ and $\vw_v$. 
Note that the weights are introduced merely as Lagrange parameters for the pairwise correlation constraints, 
not as model assumptions. 
 
The strategy $\vx^{(i)}$ which solves the inner maximization of \eqref{eq:defofdecloss} 
is to simply match signs with $\vw_v^\top \ve^{(i)} - \Gamma (\tilde{x}_v^{(i)})$ coordinate-wise for each $v \in [V]$. 
Substituting this into the above, 
\begin{align*}
\cL^* &= \frac{1}{2} \min_{ \vw_1, \dots, \vw_V } \; \lrb{ - \sum_{v=1}^V \vb_v^\top \vw_v + \frac{1}{n} \sum_{i=1}^n \min_{\vtildex^{(i)} \in [-1,1]^V } 
\sum_{v=1}^V \lrp{ \ell_{+} (\tilde{x}_v^{(i)}) + \ell_{-} (\tilde{x}_v^{(i)}) + \abs{ \vw_v^\top \ve^{(i)} - \Gamma (\tilde{x}_v^{(i)}) } } }
\\
&= \frac{1}{2} \sum_{v=1}^V \min_{ \vw_v \in \RR^H } \; \lrb{ - \vb_v^\top \vw_v + \frac{1}{n} \sum_{i=1}^n  
\min_{\tilde{x}_v^{(i)} \in [-1,1] } \lrp{ \ell_{+} (\tilde{x}_v^{(i)}) + \ell_{-} (\tilde{x}_v^{(i)}) + \abs{ \vw_v^\top \ve^{(i)} - \Gamma (\tilde{x}_v^{(i)}) } } }
\end{align*}

The absolute value breaks down into two cases, so the inner minimization's objective can be simplified:
\begin{align}
\label{eq:gencvxmmand}
\ell_{+} (\tilde{x}_v^{(i)}) + \ell_{-} (\tilde{x}_v^{(i)}) + \abs{ \vw_v^\top \ve^{(i)} - \Gamma (\tilde{x}_v^{(i)}) }
= \begin{cases} 
2 \ell_{+} (\tilde{x}_v^{(i)}) + \vw_v^\top \ve^{(i)}  \qquad & \mbox{ \; if \; } \vw_v^\top \ve^{(i)} \geq \Gamma (\tilde{x}_v^{(i)}) \\ 
2 \ell_{-} (\tilde{x}_v^{(i)}) - \vw_v^\top \ve^{(i)}  & \mbox{ \; if \; } \vw_v^\top \ve^{(i)} < \Gamma (\tilde{x}_v^{(i)})
\end{cases}
\end{align}

Suppose $\tilde{x}_v^{(i)}$ falls in the first case of \eqref{eq:gencvxmmand}, 
so that $\vw_v^\top \ve^{(i)} \geq \Gamma (\tilde{x}_v^{(i)})$. 
By definition of $\ell_{+} (\cdot)$, $2 \ell_{+} (\tilde{x}_v^{(i)}) + \vw_v^\top \ve^{(i)}$ is decreasing in $\tilde{x}_v^{(i)}$, 
so it is minimized for the greatest $\tilde{x}_v^{(i) *} \leq 1$ s.t. $\Gamma (\tilde{x}_v^{(i) *}) \leq \vw_v^\top \ve^{(i)}$. 
This means $\Gamma (\tilde{x}_v^{(i) *}) = \vw_v^\top \ve^{(i)}$, so the minimand \eqref{eq:gencvxmmand} 
is $\ell_{+} (\tilde{x}_v^{(i) *}) + \ell_{-} (\tilde{x}_v^{(i) *})$, where
$\tilde{x}_v^{i*} = \frac{1 - e^{- \vw_v^\top \ve^{(i)}} }{1 + e^{- \vw_v^\top \ve^{(i)}}}$. 


A precisely analogous argument holds if $\tilde{x}_v^{(i)}$ falls in the second case of \eqref{eq:gencvxmmand}, 
where $\vw_v^\top \ve^{(i)} < \Gamma (\tilde{x}_v^{(i)})$. 

Putting the cases together, we have shown the form of the summand $\Psi$. 
We have also shown the dependence of $\tilde{x}_v^{(i) *}$ on $\vw_v^{* \top} \ve^{(i)}$, 
where $\vw_v^{*}$ is the minimizer of the outer minimization of \eqref{eq:defofdecloss}. 
This completes the proof. 
\end{proof}

%

\subsection{$L_{\infty}$ Correlation Constraints and $L_1$ Weight Regularization}
\label{sec:linftyconstr}

Here we formalize the discussion of Sec. \ref{sec:wtreg} with the following result.

\begin{theorem}
\begin{align*}
\min_{\vtildex^{(1)}, \dots, \vtildex^{(n)} \in [-1,1]^V} &\; \max_{\substack{ \vx^{(1)}, \dots, \vx^{(n)} \in [-1,1]^V , \\ \forall v \in [V] : \;\vnorm{\frac{1}{n} \vE \vx_v - \vb_v}_{\infty} \leq \epsilon_v }} 
\; \frac{1}{n} \sum_{i=1}^n \ell (\vx^{(i)} , \vtildex^{(i)} ) \\
&= 
\frac{1}{2} \sum_{v=1}^V 
\min_{ \vw_v \in \RR^H } 
\lrb{ - \vb_v^\top \vw_v + \frac{1}{n} \sum_{i=1}^n \Psi (\vw_v^\top \ve^{(i)} ) + \epsilon_v \vnorm{\vw_v}_1 }
\end{align*}
For each $v, i$, the minimizing $\vtildex_v^{(i)}$ is a logistic function of the encoding $\ve^{(i)}$ with weights equal to the minimizing $\vw_v^*$ above, exactly as in Theorem \ref{thm:decfn}. 
\end{theorem}
\begin{proof}
The proof adapts the proof of Theorem \ref{thm:decfn}, 
following the result on $L_1$ regularization in \cite{BF16} in a very straightforward way; 
we describe this here. 

We break each $L_{\infty}$ constraint into two one-sided constraints for each $v$, i.e. 
$\frac{1}{n} \vE \vx_v - \vb_v \leq \epsilon_v \vone^n$ and $\frac{1}{n} \vE \vx_v - \vb_v \geq - \epsilon_v \vone^n$. 
These respectively give rise to two sets of Lagrange parameters $\lambda_v , \xi_v \geq \vzero^H$ for each $v$, 
replacing the unconstrained Lagrange parameters $\vw_v \in \RR^H$. 

The conditions for the minimax theorem apply here just as in the proof of Theorem \ref{thm:decfn}, 
so that \eqref{eq:defofdecloss} is replaced by 
\begin{align}
\label{eq:l1defdecloss}
\frac{1}{2} \min_{ \substack{\lambda_1, \dots, \lambda_V \\ \xi_1, \dots, \xi_V}} 
\; &\Bigg[ - \sum_{v=1}^V \lrp{ \vb_v^\top (\xi_v - \lambda_v) - \epsilon_v \vone^\top (\xi_v + \lambda_v) } \\
&+ \frac{1}{n} \min_{\tilde{\vX} } \; \sum_{i=1}^n 
\sum_{v=1}^V \lrb{ \ell_{+} (\tilde{x}_v^{(i)}) + \ell_{-} (\tilde{x}_v^{(i)}) + 
\max_{ \vx^{(i)} } \; x_v^{(i)} \lrp{ (\xi_v - \lambda_v)^\top \ve^{(i)} - \Gamma (\tilde{x}_v^{(i)}) }} \Bigg]
\end{align}

Suppose for some $h \in [H]$ that $\xi_{v,h} > 0$ and $\lambda_{v,h} > 0$. 
Then subtracting $\min(\xi_{v,h}, \lambda_{v,h})$ from both does not affect the value $[\xi_v - \lambda_v]_h$, 
but always decreases $[\xi_v + \lambda_v]_h$, and therefore always decreases the objective function. 
Therefore, we can w.l.o.g. assume that $\forall h \in [H]: \min(\xi_{v,h}, \lambda_{v,h}) = 0$. 
Defining $\vw_v = \xi_v - \lambda_v$ (so that $\xi_{v,h} = [w_{v,h}]_+$ and $\lambda_{v,h} = [w_{v,h}]_-$ for all $h$), 
we see that the term $\epsilon_v \vone^\top (\xi_v + \lambda_v)$ in \eqref{eq:l1defdecloss} can be replaced by $\epsilon_v \vnorm{\vw_v}_1$. 

Proceeding as in the proof of Theorem \ref{thm:decfn} gives the result. 
\end{proof}

\section{General Reconstruction Losses}
\label{sec:genlossfull}

Using recent techniques of \cite{BF16}, in this section we extend Theorem \ref{thm:decfn} to a larger class of reconstruction losses for binary autoencoding, 
of which cross-entropy loss is a special case. 

Since the data $\vX$ are still randomized binary, 
we first broaden the definition of \eqref{eq:defxentloss}, rewritten here:
\begin{align}
\label{eq:firstgenloss}
\ell (\vx^{(i)} , \vtildex^{(i)} ) := \sum_{v=1}^V \lrb{ \lrp{\frac{1 + x_v^{(i)} }{2} } \ell_{+} (\tilde{x}_v^{(i)} ) + \lrp{\frac{1 - x_v^{(i)} }{2} } \ell_{-} (\tilde{x}_v^{(i)} ) }
\end{align}
We do this by redefining the partial losses $\ell_{\pm} (\tilde{x}_v^{(i)} )$, 
to any functions satisfying the following monotonicity conditions.
\begin{assumption}
\label{ass:loss}
Over the interval $(-1,1)$, $\ell_{+} (\cdot)$ is decreasing and $\ell_{-} (\cdot)$ is increasing, 
and both are twice differentiable.
\end{assumption}
Assumption \ref{ass:loss} is a very natural one and includes many non-convex losses (see \cite{BF16} for a more detailed discussion, much of which applies bitwise here). 
This and the additive decomposability of \eqref{eq:firstgenloss} over the $V$ bits are the only assumptions we make on the reconstruction loss $\ell (\vx^{(i)} , \vtildex^{(i)} )$. 
The latter decomposability assumption is often natural when the loss is a log-likelihood, 
where it is tantamount to conditional independence of the visible bits given the hidden ones. 

Given such a reconstruction loss, define the increasing function $\Gamma (y) := \ell_{-} (y) - \ell_{+} (y) : [-1,1] \mapsto \RR $, 
for which there exists an increasing (pseudo)inverse $\Gamma^{-1}$. 
Using this we broaden the definition of the potential function $\Psi$:
\begin{align*}
\Psi (m) := 
\begin{cases} 
- m + 2 \ell_{-} (-1)  \qquad & \mbox{ \; if \; } m \leq \Gamma (-1) \\ 
\ell_{+} (\Gamma^{-1} (m)) + \ell_{-} (\Gamma^{-1} (m))  \qquad & \mbox{ \; if \; } m \in \lrp{ \Gamma (-1) , \Gamma (1)} \\ 
m + 2 \ell_{+} (1)  & \mbox{ \; if \; } m \geq \Gamma (1)
\end{cases}
\end{align*}

Then we may state the following result, describing the optimal decoding function for a general reconstruction loss. 
\begin{theorem}
Define the potential function 
\begin{align*}
\min_{\vtildex^{(1)}, \dots, \vtildex^{(n)} \in [-1,1]^V} &\; \max_{\substack{ \vx^{(1)}, \dots, \vx^{(n)} \in [-1,1]^V , \\ \forall v \in [V] : \;\frac{1}{n} \vE \vx_v = \vb_v }} 
\; \frac{1}{n} \sum_{i=1}^n \ell (\vx^{(i)} , \vtildex^{(i)} ) \\
&= 
\frac{1}{2} \sum_{v=1}^V 
\min_{ \vw_v \in \RR^H } 
\lrb{ - \vb_v^\top \vw_v + \frac{1}{n} \sum_{i=1}^n \Psi (\vw_v^\top \ve^{(i)} ) }
\end{align*}
For each $v \in [V], i \in [n]$, the minimizing $\vtildex_v^{(i)}$ is a sigmoid function of the encoding $\ve^{(i)}$ with weights equal to the minimizing $\vw_v^*$ above, as in Theorem \ref{thm:decfn}. 
The sigmoid is defined as 
\begin{align}
\label{eq:gipredform}
\vtildex_v^{(i)*} :=  
\begin{cases} 
-1  \qquad & \mbox{ \; if \; } \vw_v^*{^\top} \ve^{(i)} \leq \Gamma (-1) \\ 
\Gamma^{-1} (\vw_v^*{^\top} \ve^{(i)}) \qquad & \mbox{ \; if \; } \vw_v^*{^\top} \ve^{(i)} \in \lrp{ \Gamma (-1) , \Gamma (1)} \\ 
1  & \mbox{ \; if \; } \vw_v^*{^\top} \ve^{(i)} \geq \Gamma (1)
\end{cases}
\end{align}
\end{theorem}

The proof is nearly identical to that of the main theorem of \cite{BF16}. 
That proof is essentially recapitulated here for each bit $v \in [V]$ due to the additive decomposability of the loss, 
through algebraic manipulations (and one application of the minimax theorem) identical to the proof of Theorem \ref{thm:decfn} 
with the more general definitions of $\Psi$ and $\Gamma$. 
So we do not rewrite it in full here.

A notable special case of interest is the Hamming loss, for which $\ell_{\pm} (\tilde{x}_v^{(i)}) = \frac{1}{2} \lrp{1 \mp \tilde{x}_v^{(i)} }$, 
where the reconstructions are allowed to be randomized binary values. 
In this case, we have $\Psi (m) = \max(\abs{m}, 1)$, and the sigmoid used for each decoding neuron is the clipped linearity 
$\max(-1, \min(\vw_v^*{^\top} \ve^{(i)}, 1))$.

\section{Alternate Approaches}

We made some technical choices in the derivation of $\algname$, 
which prompt possible alternatives not explored here for a variety of reasons. 
Recounting these choices better explains our framework.

The output reconstructions could have restricted pairwise correlations, i.e. $\frac{1}{n} \tilde{\vX} \vE^\top = \vB$. 
One option is to impose such restrictions \emph{instead} of the existing constraints on $\vX$, leaving $\vX$ unrestricted. 
However, this is not in the spirit of this paper, because $\vB$ is our means of indirectly conveying information 
to the decoder about how $\vX$ is decoded.

Another option is to restrict both $\tilde{\vX}$ and $\vX$. 
This is possible and may be useful in propagating correlation information between layers of deeper architectures while learning, 
but its minimax solution does not have the conveniently clean structure of the $\algname$ derivation. 

In a similar vein, we could restrict $\vE$ during the encoding phase, using $\vB$ and $\vX$. 
As $\vB$ is changed only during this phase to better conform to the true data $\vX$, 
this tactic fixes $\vB$ during the optimization, which is not in the spirit of this paper's approach. 
It also performed significantly worse in our experiments.



\newpage

\end{document}